\documentclass{tdp}
\usepackage{hyperref}
\usepackage{cite}
\usepackage{url}
\usepackage[tight,footnotesize]{subfigure}
\usepackage{balance}
\usepackage{multirow}
\usepackage{fancybox}
\usepackage{mathrsfs}
\usepackage{wrapfig}
\usepackage{graphicx}
\usepackage{amsmath}
\usepackage[mathscr]{euscript}
\usepackage{algorithm}
\usepackage{algorithmic}
\usepackage{amssymb}
\usepackage{amsthm}

\theoremstyle{definition}
\newtheorem{definition}{Definition}
\newtheorem{lemma}{Lemma}

\renewcommand{\qedsymbol}{$\blacksquare$}
\newcommand{\algf}{\texttt{Maximal}}
\newcommand{\algs}{\texttt{Greedy}}
\theoremstyle{definition}
\theoremstyle{remark}

\graphicspath{{figures/}}
\usepackage{ctable}
\usepackage{bbm}
\usepackage{verbatim}
\usepackage{comment}
\usepackage{epsfig}
\usepackage{caption} 
\usepackage{fancyhdr,lipsum}
\begin{document}

\title{Feature Selection for Classification under Anonymity Constraint}

\author{Baichuan Zhang $^{*}$,
        Noman Mohammed $^{**}$, 
	Vachik S. Dave $^{*}$, 
	Mohammad Al Hasan $^{*}$}
\address{$^{*}$Department of Computer and Information Science, 
Indiana University Purdue University Indianapolis, IN, USA, 46202 \\ 
  $^{**}$Department of Computer Science, Manitoba University, Canada \\
E-mail: {\small \tt{bz3@umail.iu.edu}}, {\small \tt{noman@cs.umanitoba.ca}}, {\small \tt{vsdave@iupui.edu}}, {\small \tt{alhasan@iupui.edu}}
}


\maketitle

\begin{abstract}
Over the last decade, proliferation of various online platforms and their
increasing adoption by billions of users have heightened the privacy risk of a
user enormously. In fact, security researchers have shown that sparse microdata
containing information about online activities of a user although anonymous,
can still be used to disclose the identity of the user by cross-referencing the
data with other data sources. To preserve the privacy of a user, in existing
works several methods ($k$-anonymity, $\ell$-diversity, differential privacy)
are proposed for ensuring that a dataset bears small identity disclosure risk.
However, the majority of these methods modify the data in isolation, without
considering their utility in subsequent knowledge discovery tasks, which makes
these datasets less informative.  In this work, we consider labeled data that
are generally used for classification, and propose two methods for feature
selection considering two goals: first, on the reduced feature set the data has
small disclosure risk, and second, the utility of the data is preserved for
performing a classification task.  Experimental results on various real-world
datasets show that the method is effective and useful in practice.
\end{abstract}


\section{Introduction}

Over the last decade, with the proliferation of various online platforms, such
as web search, eCommerce, social networking, micro-messaging, streaming
entertainment and cloud storage, the digital footprint of today's Internet user
has grown at an unprecedented rate. At the same time, the availability of
sophisticated computing paradigm and advanced machine learning algorithms have
enabled the platform owners to mine and analyze tera-bytes of digital footprint
data for building various predictive analytics and personalization products.
For example, search engines and social network platforms use search keywords
for providing sponsored advertisements that are personalized for a user's
information need; e-commerce platforms use visitor's search history for
bolstering their merchandising effort; streaming entertainment providers use
people's rating data for building future product or service recommendation.
However, the impressive personalization of services of various online platforms
enlighten us as much, as they do make us feel insecure, which stems from
knowing the fact that individual's online behaviors are stored within these
companies, and an individual, more often, is not aware of the specific
information about themselves that is being stored.

The key reason for a web user's insecurity over the digital footprint data
(also known as microdata) is that such data contain sensitive information. For
instance, a person's online search about a disease medication may insinuate
that she may be suffering from that disease; a fact that she rather not want to
disclose.  Similarly, people's choice of movies, their recent purchases, etc.
reveal enormous information regarding their background, preference and
lifestyle. Arguably microdata exclude biographical information, but due to the
sheer size of our digital footprint the identity of a person can still be
recovered from these data by cross-correlation with other data sources or by
using publicly available background information. In this sense, these apparently
non-sensitive attributes can serve as a quasi-identifier. For an example, Narayanan et
al.~\cite{Narayanan.Shmatikov:08} have identified a Netflix subscriber from 
his anonymous movie rating by using Internet Movie Database (IMDB) as the 
source of background information. For the case of Netflix, anonymous
microdata was released publicly for facilitating Netflix prize competition,
however even if the data is not released, there is always a concern that
people's digital footprint data can be abused within the company by employees
or by external hackers, who have malicious intents.

For the case of microdata, the identity disclosure risk is high due to some key
properties of such a dataset---high-dimensionality and sparsity.  Sparsity
stands for the fact that for a given record there is rarely any record that is
similar to the given record considering full multi-dimensional space.  It is
also shown that in high-dimensional data, the ratio of distance to the nearest
neighbor and the farthest neighbor is almost one, i.e., all the points are far
from each other~\cite{Aggarwal:05}. Due to this fact, 
privacy is difficult to achieve on such datasets. A widely used privacy metric that quantifies the
disclosure risk of a given data instance is $k$-anonymity\cite{Samarati:01},
which requires that for any data instance in a
dataset, there are at least $k-1$ distinct data instances sharing the same
feature vector---thus ensuring that unwanted personal information cannot
be disclosed merely through the feature vector. However, for high dimensional
data, $k$-anonymization is difficult to achieve even for a reasonable value of
$k$ (say 5); typically, value based generalization or attribute based
generalization is applied so that $k$-anonymity is achieved, but Aggrawal has
proved both theoretically and experimentally that for high dimensional data
$k$-anonymity is not a viable solution even for a $k$ value of
2~\cite{Aggarwal:05}. He has also shown that as data dimensionality increases,
entire discriminatory information in the data is lost during the process of
$k$-anonymization, which severely limits the data utility.
Evidently, finding a good balance between a user's privacy
and the utility of high dimensional microdata is an unsolved problem---which is
the primary focus of this paper.

A key observation of a real-life high dimensional dataset is that it exhibits
high clustering tendency in many sub-spaces of the data, even though over the
full dimension the dataset is very sparse. Thus an alternative technique for
protecting identity disclosure on such data can be finding a subset of
features, such that when projecting on these set of features an acceptable
level of anonymity can be achieved. One can view this as column suppression
instead of more commonly used row suppression for achieving
$k$-anonymity~\cite{Samarati:01}. Now for the case of feature selection for a
given $k$, there may exist many sub-spaces for which a given dataset satisfies
$k$-anonymity, but our objective is to obtain a set of features such that
projecting on this set offers the maximum utility of the dataset in terms of a
supervised classification task.

Consider the toy dataset that is shown in Table~\ref{tab:toydata}. Each row
represents a person, and each column (except the first and the last) represents
a keyword.  If a cell entry is `1' then the keyword at the corresponding column
is associated with the person at the corresponding row. Reader may think this
table as a tabular representation of search log of an eCommerce platform, where
the `1' under $e_i$ column stands for the fact that the corresponding user has
searched using the keyword $x_j$ within a given period of time, and `0'
represents otherwise. The last column represents whether this user has made a
purchase over the same time period. The platform owner wants to solve a
classification problem to predict which of the users are more likely to make a
purchase. 

Say, the platform owner wants to protect the identity of its site visitor by
making the dataset $k$-anonymous.  Now, for this toy dataset, if he chooses
$k=2$, this dataset is not $k$-anonymous.  For instance, the feature vector of
$e_3$, $10011$ is unique in this dataset.  However, the dataset is
$k$-anonymous for the same $k$ under the subspace spanned by $\{x_3, x_4,
x_5\}$. It is also $k$-anonymous (again for the same $k$) under the subspace
spanned by $\{x_1, x_2, x_5\}$ (See Table~\ref{tab:toyfeatures}).  Among these two
choices, the latter subspace is probably a better choice, as the features in
this set are better discriminator than the features in the former set with
respect to the class-label. For feature set $\{x_1, x_2, x_5\}$, if we
associate the value `101' with the +1 label, and the value `111' with the -1
label, we make only 1 mistake out of 6. On the other hand for feature set
$\{x_3, x_4, x_5\}$, no good correlation exists between the feature values and
the class labels. 

Research problem in the above task is the selection of optimal binary feature set
for utility preserving entity anonymization, where the utility is considered
with respect to the classification performance and the privacy is guaranteed by
enforcing a $k$-anonymity like constraint~\cite{Sweeney:02}. In existing works,
$k$-anonymity is achieved by suppression or generalization of cell values,
whereas in this work we consider to achieve the same by selecting an optimal subset 
of features that maximizes the classification utility of the dataset. Note that,
maximizing the utility of the dataset is the main objective of this task, 
privacy is simply a constraint which a user enforces by setting the value of a privacy parameter 
based on the application domain and the user's judgment. For the privacy model, we 
choose $k$-anonymity by containment in short, $k$-AC (definition
forthcoming), where $k$ is the user-defined privacy parameter, which has similar meaning
as it has in traditional $k$-anonymity.

Our choice of $k$-anonymity like metric over more
theoretical counterparts, such as, differential privacy (DP) is due to the
pragmatic reason that all existing privacy laws and regulations, such as, HIPAA
(Health Information Portability and Accountability Act) and PHIPA (Personal
Health Information Protection Act) use $k$-anonymity. Also, $k$-anonymity is
flexible and simple, thus enabling people to understand and apply it for almost
any real-life privacy preserving needs; on the contrary, DP
based methods use a privacy parameter ($\epsilon$), which has no obvious
interpretation and even by the admission of original author of DP, choosing an 
appropriate value for this parameter is
difficult~\cite{Dwork:08}. Moreover, differential privacy based methods add
noise to the data entities, but the decision makers in many application domains
(such as, health care), where privacy is an important issue, are quite
uncomfortable to the idea of noise imputation~\cite{Dankar.Emam:13}. 
Finally, authors in~\cite{Fienberg.Yang.ea:10} state that 
differential privacy is not suitable
for protecting large sparse tables produced by statistics agencies and sampling 
organizations---this disqualifies differential privacy as a privacy model for protecting
sparse and very high dimensional user's microdata from the e-commerce and Internet search 
engines.

%
%

%
%
%

\subsection{Our Contributions}

In this work, we consider the task of feature selection under privacy
constraint. This is a challenging task, as it is well-known that privacy is
always at odds with the utility of a knowledge-based system, and finding the
right balance is a difficult
task~\cite{Kifer.Gehrke:06,Kifer.Machanavajjhala:11}.  Besides, feature
selection itself, without considering the privacy constraint, is an ${\cal
NP}$-Hard problem~\cite{Guyon.Elisseeff:03,li2013deploying,li2013mining}.

Given a classification dataset with binary features and an integer $k$, our proposed
solutions find a subset of features such that after projecting each instance on
these subsets each entity in the dataset satisfies a privacy constraint, called
$k$-anonymous by containment ($k$-AC). Our proposed privacy constraint $k$-AC is an
adapted version of $k$-anonymity, which strikes the correct balance between
disclosure risk and dataset utility, and it is particularly suitable for high
dimensional binary data.  We also propose two algorithms: \algf\ and \algs. The
first is a maximal itemset mining based method and the second is a greedy
incremental approach, both respecting the user-defined AC constraints.  

The algorithms that we propose are particularly intended for high dimensional
sparse microdata where the features are binary. The nature of such data is
different from a typical dataset that is considered in many of the existing
works on privacy preserving data disclosure mechanism. 
The first difference is that existing works consider two kinds of attributes, 
sensitive and nonsensitive, whereas for our dataset all attributes are considered to be sensitive, 
and any subset of published attributes can be used by an attacker to 
de-anonymize one or more entities in the dataset using probabilistic inference
methodologies. On the other hand, the unselected attributes are not published so they cannot be used
by an attacker to de-anonymize an entity.
Second, we only consider binary attributes, which enable us to provide efficient 
algorithms and an interesting anonymization model. Considering only binary attributes
may sound an undue restriction, but in reality binary attributes are adequate (and often
preferred) when modeling online behavior of a person,
such as `like' in Facebook, `bought' in Amazon, and `click' in Google
advertisement. Also, collecting explicit user feedback in terms of
frequency data (say, the number of times a search keyword is used )
may be costly in many online platforms. 
Nevertheless, as shown
in~\cite{Narayanan.Shmatikov:08}, binary attributes are sufficient for an
attacker to de-anonymize a person using high dimensional microdata, so safeguarding
user privacy before disclosing such dataset is important.

The contributions of our work are outlined below:

\begin{enumerate}

\item We propose a novel method for entity anonymization using feature selection
over a set of binary attributes from a two-class classification dataset.
For this, we design a new anonymization model, named {\em $k$-anonymity by containment ($k$-AC)},
which is particularly suitable for high-dimensional binary microdata. 

\item We propose two methods for solving the above task and show experimental
results to validate the effectiveness of these methods.

\item We show the utility of the proposed methods with three real-life
applications; specifically, we show how the privacy-aware feature selection
affects their performance.
\end{enumerate}

\section{Privacy Basics}

Given a dataset $D$, where each row corresponds to a
person, and each column contains 
non-public information about that person; examples include disease, medication,
sexual orientation, etc. In the context of online behavior, the search
keywords, or purchase history of a person may be such information.
Privacy preserving data publishing methodologies make it difficult for an
attacker to de-anonymize the identity of a person who is in the dataset. For
de-anonymization, an attacker generally uses a set of attributes that act
almost like a key and it uniquely identifies some individual in the datasets.
These attributes are called {\em quasi-identifiers}.
{\em $k$-anonymity} is a well-known privacy metric defined as below.

\begin{definition}[$k$-anonymity] {\em A dataset $D$ satisfies $k$-anonymity if
for any row entity $e \in D$ there exist at least $k-1$ other entities that have
the same values as $e$ for every possible quasi-identifiers.}
\end{definition}

The database in Table~\ref{tab:toydata} is not $2$-anonymous, as the row entity
$e_3$ is unique considering the entire attribute-set $\{x_1, x_2, x_3,x_4,
x_5\}$ as quasi-identifier. On the other hand, It is $2$-anonymous for both the
datasets (one with Feature Set-1 and the other with Feature Set-2) in
Table~\ref{tab:toyfeatures}. For numerical or categorical attributes, a
process, called {\em generalization} and/or suppression (row or cell value) are
used for achieving $k$-anonymity. Generalization partitions the values of an
attribute into disjoint buckets and identifies each bucket with a value.
Suppression either hides the entire row or some of its cell values, so that the
anonymity of that entity can be maintained. Generalization and suppression make
anonymous group, where all the entities in that group have the same value for
every possible quasi-identifier, and for a dataset to be $k$-anonymous, the
size of each of such groups is at least $k$.  In this work we consider binary
attributes; each such attribute has only two values, 0 and 1. For binary
attributes, value
based generalization relegates to the process of column suppression, which incurs 
a loss of data utility. In fact,
any form of generalization based $k$-anonymization incurs a loss in data utility due to the
decrement of data variance or due to the loss of discernibility.  Suppression
of a row is also a loss as in this case the entire row entity is not
discernible for any of the remaining entities in the dataset.
Unfortunately, most of existing methods for achieving $k$-anonymity using
both generalization and suppression operations do not consider an utility measure targeting 
supervised classification task.

There are some security attacks against which $k$-anonymity is vulnerable. For
example, $k$-anonymity is susceptible to both homogeneity and background
knowledge based attacks. More importantly, $k$-anonymity does not provide
statistical guaranty about anonymity which can be obtained by using
$\epsilon$-differential privacy~\cite{Dwork:08}---a method which provides
strong privacy guarantees independent of an adversary's background knowledge.
There are existing methods that adopt differential privacy principle for
data publishing. Authors in~\cite{Barak.Dwork.ea:07,Dwork.Frank.ea:06} propose Laplace mechanism to
publish the contingency tables by adding noise generated from a Laplace
distribution. However, such methods suffer from the utility loss due to the
large amount of added noise during the sanitization process. To
resolve this issue, ~\cite{Frank.Kunal:07} proposes to utilize exponential
mechanism for maximizing the trade-off between differential privacy and data
utility. However, the selection of utility function used in the exponential mechanism based 
approach strongly affects the data utility in subsequent data analysis task. 
In this work, we compare the performance of our proposed privacy model, namely $k$-AC, 
with both Laplace and exponential based differential privacy frameworks (See Section~\ref{sec:5.2})
to show that $k$-AC better preserves the data utility than the differential privacy based methods.

A few works~\cite{Iyengar:02,Prasser:2016:LUA:2993206.2993209} exist which consider classification utility together with $k$-anonymity
based privacy model, but none of them consider feature selection which is the main
focus of this work. In one of the earliest works, 
Iyengar~\cite{Iyengar:02} solves $k$-anonymization through generalization and suppression while
minimizing a proposed utility metric called $CM$ (Classification Metric) using
a genetic algorithm which provides no optimality guaranty. The CM is defined as below:

\begin{definition}[Classification Metric~\cite{Iyengar:02}]~\label{def:cm}{\em Classification
metric ($CM$) is a utility metric for classification dataset, which assigns a
penalty of 1 for each suppressed entity, and for each non-suppressed entity it
assigns a penalty of 1 if those entities belong to the minority class within
its anonymous group. $CM$ value is equal to the sum of penalties over all the
entities.} \end{definition}

In this work, we compare the performance of our work with CM based privacy-aware
utillty metric.

\begin{table}
\centering
\begin{tabular} {l c c c c c l}\hline
\toprule
User & $x_1$ & $x_2$ & $x_3$ & $x_4$ & $x_5$ & Class\\ \midrule
$e_1$ &  1    &   0   &   1   &   0   &   1   &   $+1$\\ 
$e_2$ &  1    &   0   &   1   &   0   &   1   &   $-1$\\
$e_3$ &  1    &   0   &   0   &   1   &   1   &   $+1$\\
$e_4$ &  1    &   0   &   1   &   0   &   1   &   $+1$\\
$e_5$ &  1    &   1   &   1   &   0   &   1   &   $-1$\\
$e_6$ &  1    &   1   &   0   &   1   &   1   &   $-1$\\ \bottomrule
\end{tabular}
\caption{A toy 2-class dataset with binary feature-set}
\label{tab:toydata}
\end{table}

\begin{table}
\centering
\begin{tabular} {l c c c c c c c}
\toprule
\multirow{2}{*}{User} &  \multicolumn{3}{c}{Feature Set-1} & \multirow{2}{*}{Class} & 
\multicolumn{3}{c}{Feature Set-2}  \\  \cline{2-4} \cline{6-8}
{} & $x_1$ & $x_2$ & $x_5$ & {} & $x_3$ & $x_4$ & $x_5$ \\ \midrule
$e_1$ & 1 & 0 & 1 & $+1$ & 1 & 0 & 1\\ 
$e_2$ & 1 & 0 & 1 & $-1$ & 1 & 0 & 1\\ 
$e_3$ & 1 & 0 & 1 & $+1$ & 0 & 1 & 1\\ 
$e_4$ & 1 & 0 & 1 & $+1$ & 1 & 0 & 1\\ 
$e_5$ & 1 & 1 & 1 & $-1$ & 1 & 0 & 1\\ 
$e_6$ & 1 & 1 & 1 & $-1$ & 0 & 1 & 1\\ \bottomrule
\end{tabular}
\caption{Projections of the dataset in Table~\ref{tab:toydata} on two feature-sets (Feature Set-1
and Feature Set-2)}
\label{tab:toyfeatures}
\end{table}

\section{Problem Statement}\label{sec:problemstatement}

Given a classification dataset with binary attributes, our objective is to find
a subset of attributes which increase the non-disclosure protection of the row
entities, and at the same time maintain the classification utility of the
dataset without suppressing any of the row entities. In this section we will
provide a formal definition of the problem. 

We define a database $D (E, I)$ as a binary relation between a set of entities ($E$) and a
set of features ($I$); thus, $D \subseteq E \times I$, where $E = \{e_{1}, e_{2},
\cdots, e_{n}\}$ and $I = \{x_{1}, x_{2}, \cdots, x_{d}\}$; $n$ and $d$ are the
number of entities and the number of features, respectively. The database $D$
can also be represented as a $n \times d$ binary data matrix,  where the rows
correspond to the entities, and the columns correspond to the features. For an
entity $e_i \in E$, and a feature $x_j \in I$, if $\langle e_i, x_j \rangle \in D$, the
corresponding data matrix entry $D(e_i,x_j) = 1$, otherwise $D(e_i,x_j)=0$. Thus each
row of $D$ is a binary vector of size $d$ in which the $1$ entries correspond
to the set of features with which the corresponding row entity is associated.
In a classification dataset, besides the attributes, the entities are also
associated to a class label which is a category value. In this task we assume a
binary class label $\{C_1,C_2\}$. A typical supervised learning task is to use the 
features $I$ to predict the class label of an entity. 

We say that an entity $e_i \in E$ contains a set of features $X = \{x_{i1},
x_{i2}, \cdots, x_{il}\}$, if $D(e_i, x_{ik}) = 1$ for all $k = 1, 2,
\cdots, l$; set $X$ is also called {\em containment set} of the entity $e_i$.

\begin{definition}[Containment Set] {\em Given a binary dataset, $D(E, I)$, the
containment set of a row entity $e \in E$, represented as $CS_D(e)$,  is the
set of attributes $X \subseteq I$ such that $\forall x \in X, D(e,x)=1$, and $\forall y \in I-X$,
$D(e, y) = 0$.}
\end{definition}

When the dataset $D$ is clear from the context we will simply write $CS(e)$
instead of $CS_D(e)$ to represent the containment set of $e$.

\begin{definition}[$k$-anonymity by containment] \label{def:kca} {\em In a binary
dataset $D(E, I)$ and for a given positive integer $k$, an entity $e \in E$ satisfies
$k$-anonymity by containment if there exists a set of entities $F \subseteq E$,
such that $ e \notin F \wedge |F|\ge k-1 \wedge \forall f \in F,  CS_D(f)
\supseteq CS_D(e)$. In other words, their exist at least $k-1$ other entities
in $D$ such that their containment set is the same or a superset of $CS_D(e)$.}
\end{definition}

By definition, if an entity satisfies $k$-anonymity by containment, it satisfies
the same for all integer values from 1 upto $k$. We use the term $AC(e)$ to
denote the largest $k$ for which the entity $e$ satisfies $k$-anonymity by containment. 

\begin{definition}[$k$-anonymous by Containment Group] {\em For a binary dataset
$D(E, I)$, if $e \in E $ satisfies the $k$-anonymity by containment,
$k$-anonymous by containment group with respect to $e$\ exists and this is $F
\cup \{e \}$, where $F$ is the largest possible set as is defined in
Definition~\ref{def:kca}.} \end{definition}

\begin{definition}[$k$-anonymous by Containment Dataset] {\em A binary dataset
$D(E, I)$ is $k$-anonymous by containment if every entity $e \in E$ satisfies
$k$-anonymity by containment.}
\end{definition}

We extend the term $AC$ over a dataset as well, thus $AC(D)$ is the smallest $k$ for which 
the dataset $D$ is anonymous.

\noindent {\bf Example:} For the dataset in Table~\ref{tab:toydata},
$CS(e_1)=\{x_1,x_3,x_5\}$.  Entity $e_1$ satisfies $4$-anonymity by containment,
because for each of the following three entities $e_2, e_4$, and $e_5$, their
containment sets are the same or supersets of $CS(e_1)$. But, the entity $e_6$ only
satisfies $1$-anonymity by containment, as besides itself no other entity
contains $CS(e_6)=\{x_1, x_2, x_4, x_5\}$. $4$-anonymous by containment group of
$e_2$ exists, and it is $\{e_1, e_2, e_4, e_5\}$, but $5$-anonymous by
selection group for the same entity does not exist. The dataset in
Table~\ref{tab:toydata} is $1$-anonymous by containment because there exists one
entity, namely $e_6$ such that the highest $k$-value for which $e_6$ satisfies
$k$-anonymity by containment is 1; alternatively $AC(D)=1$ \qedsymbol

$k$-anonymity by containment ($k$-AC) is the privacy metric that we use in this work. The
argument for this metric is that if a large number of other entities contain
the same or super feature subset which an entity $e$ contains, the disclosure protection
of the entity $e$ is strong, and vice versa. Thus a higher value of $k$ stands
for a higher level of privacy for $e$. $k$-anonymity by containment ($k$-$AC$) is similar to
$k$-anonymity for binary feature set except that for $k$-$AC$ only the `1'
value of feature set is considered as a privacy risk. It is easy 
to see that $k$-anonymity by containment ($k$-AC) is a relaxation of $k$-anonymity. 
In fact, the following lemma holds.

\begin{lemma}
If a dataset satisfies $k$-anonymity for 
a $k$ value, it also satisfies $k$-AC for the same $k$-value, but the reverse does 
not hold necessarily.
\end{lemma}

\begin{proof}
Say, the dataset $D$ satisfies $k$-anonymity; then for any row entity $e \in D$, there
exists at least $k-1$ other row entities with identical row vector as $e$. Containment set of
all these $k-1$ entities is identical to $e$, so $e$ satisfies $k$-AC. Since this holds
for all $e \in D$, the dataset $D$ satisfies $k$-AC. 

To prove that the reverse does not hold, we will give a counter-example. Assume $D$ has
three entities and two features with the following feature values, $D = \{(1, 1), (1, 0), (1, 1)\}$. 
$D$ satisfies $2$-AC because the smallest anonymity by containment value of the 
entities in the dataset is 2. But, the dataset does not satisfy 2-anonymity
because the entity $(1,0)$ is unique in the dataset. 
\end{proof}

However, the relaxed privacy that $k$-AC provides is adequate
for disclosure protection in a high dimensional sparse microdata with binary
attributes, because $k$-AC conceals
the list of the attributes in a containment set of an entity, which could reveal
sensitive information about the entity. For example, if the dataset is about
the search keywords that a set of users have used over a given time, for a
person having a 1 value under a keyword potentially reveals sensitive information about
the behavior or preference of that person. Having a value of 0 for a collection of features
merely reveals the knowledge that the entity is {\em not associated}
with that attribute. In the online microdata domain, due to the high dimensionality
of the data, non-association with a set of attributes is not a potential privacy risk.
Also note that, in traditional datasets, only a few attributes
which belong to non-sensitive group are assumed to be quasi-identifier, so a
privacy metric, like $k$-anonymity works well for such dataset. But, for high-dimensional
dataset, $k$-anonymity is severely restrictive and utility loss of data by column
suppression is substantial because feature subsets containing very small number of features
pass $k$-anonymity criteria. On the other hand, $k$-AC based privacy metric enables
selection of sufficient number of features for retaining the classification utility of
the dataset. In short, $k$-AC retains the classification utility substantially, whereas 
$k$-anonymity fails to do so for most high dimensional data.

Feature selection~\cite{Guyon.Elisseeff:03} for a classification task is to select a subset of highly
predictive variables so that classification accuracy possibly improves which
happens due to the fact that contradictory or noisy attributes are generally
ignored during the feature selection step. For a dataset $D(E, I)$, and a
feature-set $S \subseteq I$, following relational algebra notations, we use
$\Pi_S(D)$ to denote the projection of database $D$ over the feature set $S$.
Now, given a user-defined integer number $k$, our goal is to perform an optimal
feature selection on the dataset $D$ to obtain $\Pi_S(D)$ which satisfies two
objectives: first, $\Pi_S(D)$ is $k$-anonymous by containment, i.e.,
$AC(\Pi_S(D))\ge k$; second, $\Pi_S(D)$ maintains the predictive performance of
the classification task as much as possible. Selecting a subset of features is
similar to the task of column suppression based privacy protection, but the challenge
in our task is that we want to suppress column
that are risk to privacy, and at the same time we want to retain columns that have 
good predictive performance for a downstream supervised classification task using the sanitized dataset.
For denoting the predictive
performance of a dataset (or a projected dataset) we define a classification
utility function $f$. The higher the value of $f$, the better the dataset for the
classification. We consider $f$ to be a filter based feature selection criteria
which is independent of the classification model that we use.

The formal research task of this work is as below. Given a binary dataset $D (E, I)$, and 
an integer number $k$, find $S \subseteq I$ so that $f(\Pi_S(D))$ is maximized
under the constraint that $AC(\Pi_S(D)) \ge k$. Mathematically,

\begin{equation}\label{eq:opt}
\begin{aligned}
& \underset{S \subseteq I}{\text{maximize}}
& & f(\Pi_S(D))\\
& \text{subject to}
& & AC(\Pi_S(D)) \ge k\\
\end{aligned}
\end{equation}

Due to the fact that the problem~\ref{eq:opt} is a combinatorial optimization problem (optimizing over
the space of feature subsets) which is {\cal NP}-Hard, here 
we propose two effective local optimal solutions for this problem.

\section{Methods}

In this section, we describe two algorithms, namely \algf\ and \algs\ that
we propose for the task of feature selection under privacy constraint. \algf\
is a maximal itemset mining based feature selection method, and \algs\ is a greedy
method with privacy constraint based filtering. 
In the following subsections, we discuss them in details.

\subsection{Maximal Itemset based Approach}\label{sec:1}

A key observation regarding $k$-anonymity by containment ($AC$) of a dataset is
that this criteria satisfies the downward-closure property under feature
selection. The following lemma holds:

\begin{lemma} \label{lemma:dc}
Say $D(E,I)$ is a binary dataset and $X \subseteq I$ and $Y \subseteq I$ are two feature subsets. 
If $X \subseteq Y$, then $AC(\Pi_X(D)) \ge AC(\Pi_Y(D))$.

{\noindent}{\sc Proof}: Let's prove by contradiction. Suppose $X \subseteq Y$
and $AC(\Pi_X(D)) < AC(\Pi_Y(D))$.  Then from the definition of $AC$, there
exists at least one entity $e \in E$ for which $AC(\Pi_X(e)) < AC(\Pi_Y(e))$.
Now, let's assume $A_X$ and $A_Y$ are the set of entities which make the
anonymous by containment group for the entity $e$ in $\Pi_X(D)$ and $\Pi_Y(D)$,
respectively.  Since $AC(\Pi_X(e)) < AC(\Pi_Y(e)), |A_X| < |A_Y|$; so there
exists an entity $p \in A_Y \setminus A_X$, for which 
$CS_{\Pi_Y(D)}(p) \supseteq CS_{\Pi_Y(D)}(e)$ and $CS_{\Pi_X(D)}(p) \not\supseteq CS_{\Pi_X(D)}(e)$; 
But this is impossible, because $X \subseteq Y$, if $CS_{\Pi_Y(D)}(p) \supseteq CS_{\Pi_Y(D)}(e)$ holds,
then $CS_{\Pi_X(D)}(p) \supseteq CS_{\Pi_X{D}}(e)$ must be true.   
Thus, the lemma is proved by contradiction. \qed.  
\end{lemma}

Let's call the collection of feature subsets 
which satisfy the $AC$ threshold for a given $k$, the feasible set and represent
it with ${\cal F}_k$. Thus,
${\cal F}_k = \{X \mid X \subseteq I \wedge AC(\Pi_X(D)) \ge k \}$. A subset of 
features $X \in {\cal F}_k$ is called maximal if it has no supersets which is feasible.
Let ${\cal M}_k$ be the set of all maximal subset of features. Then
${\cal M}_k = \{X \mid X \in {\cal F}_k \wedge \not\exists Y \supset X, \text{such that} 
      \enspace Y \in {\cal F}_k\}$.
As we can observe given an integer $k$, if there exists a maximal feature set
$Z$ that satisfies the $AC$ constraint, then any feature set $X \subseteq Z$,
also satisfies the same $AC$ constraint, i.e., $k \le AC(\Pi_Z(D)) \le AC(\Pi_X(D))$
if $X \subseteq Z \in {\cal M}_k$ based on the Lemma~\ref{lemma:dc}.

\noindent {\bf Example}: For the dataset in Table~\ref{tab:toydata},
the $2$-anonymous by containment feasible feature set~\footnote{To enhance the readability, we write the
feature set as string; for example, the set $\{x_1, x_2\}$ is written as $x_1x_2$.} 
 ${\cal F}_2= \{x_1, x_2, x_3,
x_4, x_5, x_1x_2,x_1x_3,x_1x_4, x_1x_5, x_2x_5, x_3x_5, x_4x_5, x_1x_2x_5,x_1x_3x_5,x_1x_4x_5\}$ 
and ${\cal M}_2 = \{x_1x_2x_5, x_1x_3x_5, x_1x_4x_5\}$.  
In this dataset, the feature-set $x_2x_3 \notin {\cal F}_2$ because
in $\Pi_{x_2x_3}(D)$, $CS(e_5) = \{x_2, x_3\}$ and the size of the $k$-anonymous
by containment group of $e_5$ is 1; thus $AC(\Pi_{x_2x_3}(D)) = 1 < 2$. On the
other hand for feature-set $x_1x_2x_5$, the projected dataset $\Pi_{x_1x_2x_5}(D)$
has two $k$-anonymous by containment groups, which are $\{e_1,e_2,e_3,e_4,e_5, e_6\}$ and
$\{e_5, e_6\}$; since each group contains at least two entities, \\ $AC(\Pi_{x_1x_2x_5}(D)) =2$ $\qed$

\begin{lemma}\label{lemma:itemsetmine}
Say, $D(E,I)$ is a binary dataset, and $T$ is its transaction representation where
each entity $e \in E$ is a transaction consisting of the containment set $CS_D(e)$.
Frequent itemset of the dataset $T$ with minimum support threshold $k$ are the
feasible feature set ${\cal F}_k$ for the optimization problem~\ref{eq:opt}.

{\noindent}{\sc Proof}: Say, $X$ is a frequent itemset in the transaction $T$
for support threshold $k$. Then, the support-set of $X$ in $T$ are the
transactions (or entities) which contain $X$. Since, $X$ is frequent, the
support-set of $X$ consists of at least $k$ entities. In the projected
dataset $\Pi_X(D)$, all these entities make a $k$-anonymous by containment
group, thus satisfying $k$-anonymity by containment. For each of the remaining
entities (say, $e$), $e$'s containment set contains some subset of $X$ (say $Y$) 
in $\Pi_X(D)$.  Since, $X$ is a frequent itemset and $Y \subset X$, $Y$ is also 
frequent with a support-set that has at least $k$ entities. Then $e$ also belongs
to a $k$-anonymous by containment group. Thus, each of the
entities in $D$ belongs to some $k$-anonymous by containment group(s) which yields:
$X~\text{is frequent} \Rightarrow X \in {\cal F}_k$. Hence proved. \qed

\end{lemma}

A consequence of Lemma~\ref{lemma:dc} is that for a given dataset $D$, an integer $k$, and
a feature set $S$, if $AC(\Pi_S(D))\ge k$, any subset of $S$ (say, $R$)
satisfies $AC(\Pi_R(D))\ge k$. This is identical to the downward closure
property of frequent itemset mining. Also, Lemma~\ref{lemma:itemsetmine}
confirms that any itemset that is frequent in the transaction representation of
$D$ for a minimum support threshold $k$ is a feasible solution
for problem~\ref{eq:opt}.  Hence, Apriori like algorithm for itemset
mining can be
used for effectively enumerating all the feature subsets of $D$ which satisfies
the required $k$-anonymity by containment constraint.

\subsubsection{Maximal Feasible Feature Set Generation}

For large datasets, the feasible feature set ${\cal F}_k$ which consists of
feasible solutions for the optimization problem~\ref{eq:opt} can be very large.
One way to control its size is by choosing appropriate $k$; if $k$ increases,
$|{\cal F}_k|$ decreases, and vice-versa, but choosing a large $k$ negatively
impacts the classification utility of the dataset, thus reducing the optimal
value of problem (\ref{eq:opt}). A brute force method for finding the optimal
feature set $S$ is to enumerate all the feature subset in ${\cal F}$ and find
the one that is the best given the utility criteria $f$.  However, this can be
very slow. So, \algf\ generates all possible maximal feature sets ${\cal M}_k$
instead of generating ${\cal F}_k$ and search for the best feature subset
within ${\cal M}_k$. The idea of enumerating ${\cal M}_k$ instead of ${\cal
F}_k$ comes from the assumption that with more features the classification
performance will increase; thus, the
size of the feature set is its utility function value; i.e., $f(\Pi_S(D)) =
|S|$, and in that case the largest set in ${\cal M}_k$ is the solution to the
problem~\ref{eq:opt}.  

An obvious advantage of working only
with the maximal feature set is that for many datasets, $|{\cal M}_k| << |{\cal
F}_k|$, thus finding solution within ${\cal M}_k$ instead of ${\cal F}_k$ leads
to significant savings in computation time.  
Just like the case for frequent itemset mining, maximal frequent itemset mining
algorithm can also be used for finding ${\cal M}_k$. Any off-the-shelf software
can be used for this. In \algf algorithm we use the LCM-Miner package provided
in\footnote{\url{http://research.nii.ac.jp/~uno/code/lcm.html}} which, at present, 
is the fastest method for finding maximal frequent itemsets. 

\subsubsection{Classification Utility Function}

The simple utility function $f(\Pi_S(D)) = |S|$ has a few limitations. First,
the ties are very commonplace, as there are many maximal feature sets that have
the same size. Second, and more importantly, this function does not take into
account the class labels of the instances so it cannot find a feature set that
maximizes the separation between the positive and negative instances. So, we
consider another utility function, named as $HamDist$, which does not succumb much to
the tie situation. It also considers the class label for choosing features
that provide good separation between the positive and negative classes. 

\begin{definition}[Hamming Distance]
{\em For a given binary database
$D(E,I)$, and a subset of features, $S \subseteq I$, the Hamming distance
between $\Pi_S(a)$ and $\Pi_S(b)$ is defined as below:}
\begin{equation}
d_{H}(\Pi_S(a), \Pi_S(b))
=  \displaystyle \sum_{j=1}^{\mid S \mid} \mathbbm{1}\{a_{s_j} \neq b_{s_j}, \forall s_{j} \in S\} 
\label{eq:hamming}
\end{equation}
{\em where $\mathbbm{1}\{X\}$ is the indicator function, and $a_{s_j}$ and $b_{s_j}$
are the $s_j$th feature value under $S$ for the entities $a$ and $b$, respectively.}
\end{definition}

We can partition the entities in $D(E, I)$ into two disjoint subsets, $E_1$ and $E_2$;
entities in $E_1$ have a class label value of $C_1$, and entities in $E_2$ have a class 
label value of $C_2$.

\begin{definition}[{\em HamDist}]~\label{def:hamdist} {\em Given a dataset $D(E=E_1 \cup E_2,I)$ where the partitions
$E_1$ and $E_2$ are based on class labels,
the classification utility function $HamDist$ for a feature 
subset $S \subseteq I$ is the
average Hamming distance between all pair of entities $a$ and $b$ such that $a \in E_1$
and $b \in E_2$.}
\begin{equation}
HamDist(S) = \frac{1}{\mid E_{1} \mid \mid E_{2} \mid} \displaystyle \sum_{a \in E_{1}, b \in E_{2}} 
d_H(\Pi_S(a),\Pi_S(b))
\label{eq:averagehamming}
\end{equation}
\end{definition}

\noindent{\bf Example:} For the dataset in Table~\ref{tab:toydata}, for its projection
on $x_1x_2x_5$ (see, Table 2), distance of $e_1$ from the negative entities are $0+1+1=2$,
and the same for the other positive entities, $e_3$ and $e_4$ also. So, 
$HamDist(x_1x_2x_5) = ((0+1+1)+(0+1+1)+(0+1+1))/9=6/9$. From the same table we can also
see that $HamDist(x_3x_4x_5)=8/9$.$\qed$

As we can observe from Equations~\ref{eq:hamming} and~\ref{eq:averagehamming},
the utility function $HamDist(S)$ reflects the discriminative power between
classes given the feature set $S$. The larger the value of $HamDist(S)$, the
better the quality of selected feature set $S$ to distinguish between classes.
Another separation metric similar to $HamDist$ is {\em DistCnt} (Distinguish Count), 
which is defined below.

\begin{definition}[{\em DistCnt}]~\label{def:distcnt} {\em For $D(E_1 \cup E_2, I)$, and $S \subseteq I$,
$DistCnt$ is the number of 
pairs from $E_1$ and $E_2$ which can be distinguished using at least one feature in $S$. Mathematically,}
\begin{equation}
DistCnt(S) = \frac{1}{\mid E_{1} \mid \mid E_{2} \mid} \displaystyle \sum_{a \in E_{1}, b \in E_{2}} 
{\mathbbm{1}\{\Pi_S(a) \neq \Pi_S(b)\}}
\label{eq:distcnt}
\end{equation}
\end{definition}
$DistCnt$ can also be used instead of $HamDist$ in the \algf\ algorithm. 
Note that, we can also use $CM$ criterion (see Definition~\ref{def:cm}) 
instead of $HamDist$; however, experimental results
show that $CM$ performs much poorer in terms of AUC.
Besides, both $HamDist$ and $DistCnt$ functions have some good properties (will be discussed in
Section ~\ref{sec:2}) which $CM$ does not have.

The \algf\ algorithm utilizes classification utility metrics ($HamDist$ or
$DistCnt$) for selecting the best feature set from the maximal set ${\cal
M}_k$. For some datasets, the size of ${\cal M}_k$ can be large and selecting
the best feature set by applying the utility metric on each element of ${\cal
M}_k$ can be time-consuming. Then, we can find the best feature set among the
largest sized element in ${\cal M}_k$. Another option is to consider the
maximal feature sets in ${\cal M}_k$ in the decreasing order of their size in
such a way that at most $r$ of the maximal feature sets from set ${\cal M}_k$
are chosen as candidate for which the utility metric computation is performed.
In this work we use this second option by setting $r=20$ for all our
experiments. 

\begin{algorithm}
\renewcommand{\algorithmicrequire}{\textbf{Input:}}
\renewcommand{\algorithmicensure}{\textbf{Output:}}
\caption{Maximal Itemset Mining Based Feature Selection Method}
\label{alg1}
\begin{algorithmic}[1]
\REQUIRE $D(E,I)$, $k$, $r$
\ENSURE $S$
\STATE Calculate maximal feature set ${\cal M}_k$ which contains the feature-sets
satisfying $k$-AC for the given $k$
\STATE Select best feature-set $S$ based on the $HamDist$ criteria by considering
$r$ largest-sized feature set in ${\cal M}_k$.\\
\STATE \textbf{return} $S$
\end{algorithmic}
\end{algorithm}  

\subsubsection{Maximal Itemset based Method (Pseudo-code)}

The pseudo-code for \algf\ is given in Algorithms~\ref{alg1}.
\algf\ takes integer number $k$ and the number of maximal patterns $r$ as
input and returns the final feature set $S$ which satisfies $k$-anonymity by containment.
Line 1 uses the LCM-Miner to generate all the maximal
feature sets that satisfy $k$-anonymity by containment for the given $k$ value. Line 2 groups maximal
feasible feature sets according to its size and selects top $r$ maximal feature sets
with the largest size and builds the candidate feature sets. Then the algorithm computes
the feature selection criteria $HamDist$ of each feature set in the candidate feature sets and returns the
best feature set that has the maximum value for this criteria.

The complexity of the above algorithm predominantly depends on the complexity
of the maximal itemset mining step (Line 1), which depends on the input value $k$.  
For larger $k$, the privacy is stronger and it reduces
${\cal M}_k$ making the algorithm run faster, but the classification utility
of the dataset may suffer. On the other hand, for smaller $k$, ${\cal M}_k$ can
be large making the algorithm slower, but it better retains the classification
utility of the dataset.

\subsection{Greedy with Modular and Sub-Modular Objective Functions}\label{sec:2}

A potential limitation of \algf\ is that for dense datasets this method can be
slow. So, we propose a second method, called \algs\ which runs much faster as
it greedily adds a new feature to an existing feasible feature-set. For greedy
criteria, \algs\ can use different separation functions which discriminate
between positive and negative instances. In this work we use $HamDist$ (See
Definition~\ref{def:hamdist}) and $DistCnt$ (See Definition~\ref{def:distcnt}). 
Thus \algs\ solves the
Problem (\ref{eq:opt}) by replacing $f$ by either of the two functions. Because
of the monotone property of these functions, \algs\ ensures that as we add more
features, the objective function value of (\ref{eq:opt}) monotonically
increases. The process stops once no more features are available to add to the
existing feature set while ensuring the desired $AC$ value of the projected
dataset. 

\subsubsection{Submodularity, and Modularity}
\begin{definition} [Submodular Set Function] {\em Given a finite ground set $U$, 
a monotone function $f$ that maps subsets of $U$ to a real number $f:2^{U} \rightarrow \mathbb{R}$ 
is called submodular if}
\begin{equation*}\label{eq:sm}
f(S \cup \{u\}) - f(S) \ge f(T \cup \{u\}) - f(T), \forall S \subseteq T \subseteq U, u \in U
\end{equation*}
{\em If the above condition is satisfied with equality, the function is called modular.}
\end{definition}

\begin{theorem}
$HamDist$ is monotone, submodular, and modular.
\end{theorem}

\begin{proof} For a dataset $D(E,I)$, $S \subseteq I$, and $T \subseteq I$ are two arbitrary
feature-sets, such that $S \subseteq T$.  $E=E_1 \cup E_2$, where the partition
is based on class label. Consider the pair $(a,b)$, such that $a \in E_1$ and
$b \in E_2$. Let, $w(\cdot)$ be a function that sums the Hamming distance over
all such pairs $(a,b)$ for a given feature subset $S$. Thus, $w(S) = \sum_{a
\in E_1, b \in E_2}{d_{H}(\Pi_S(a), \Pi_S(b))}$, where the function $d_H$ is
the Hamming distance between $a$ and $b$ as defined in
Equation~\ref{eq:hamming}.  Similarly we can define $w(T)$, for the
feature subset $T$. Using Equation~\ref{eq:hamming}, $d_{H}(\Pi_S(a),
\Pi_S(b))$ is the summation over each of the features in $S$. Since $S
\subseteq T$,  $d_{H}(\Pi_T(a), \Pi_T(b))$ includes the sum values for the
variables in $S$ and possibly includes the sum value of other variables, which
is non-negative. Summing over all ($a, b$) pairs yields $w(S)\le w(T)$. So,
$HamDist$ is monotone. Now, for a feature $u \notin T$,

\begin{align*}
w(S \cup \{u\}) & = \sum_{a \in E_1, b \in E_2}{d_{H}\left(\Pi_{S\cup \{u\}}(a), 
                    \Pi_{S\cup \{u\}}(b)\right)}\\
                & = \sum_{a \in E_1, b \in E_2}{\sum_{s_j \in S \cup \{u\}}
                    {\mathbbm{1}\{a_{s_j} \neq b_{s_j}\}}} \text{(using Eq.~\ref{eq:hamming})}\\
                & = \sum_{a \in E_1, b \in E_2}{\left(\sum_{s_j \in S }
                    {\mathbbm{1}\{a_{s_j} \neq b_{s_j}\}} + 
                    {\mathbbm{1}\{a_{u} \neq b_{u}\}}\right)}\\
                & = w(S) + w(\{u\})
\end{align*}
Similarly, $w(T \cup \{u\}) = w(T)+w(\{u\})$. Then, we have
$w(\{u\}) = w(S \cup \{u\}) - w(S) = w(T \cup \{u\}) - w(T)$. Dividing both sides by
$1/(|E_1|\cdot |E_2|)$ yields $HamDist(S \cup \{u\}) - HamDist(S) = HamDist(T \cup\{u\}) - HamDist(T)$. Hence proved with the equality.
\end{proof}

\begin{theorem}$DistCnt$ is monotone, and submodular. \end{theorem}
\begin{proof} Given a dataset $D(E,I)$ where $E$ is partitioned as $E_1 \cup E_2$ based on class
label. Now, consider a bipartite graph, where vertices in one partition (say, $V_1$)
correspond to features in $I$, and the vertices of other partition (say, $V_2$)
correspond to a distinct pair of entities $(a,b)$ such that $a \in E_1$,
and $b \in E_2$; thus, $|V_2|=|E_1| \cdot |E_2|$. If for a feature  $x
\in V_1$, we have $a_x \ne b_x$, an edge exists between the corresponding
vertices $x \in V_1$ and $(a,b) \in V_2$. Say, $S \subseteq V_1$ and  $T \subseteq V_1$
and $S \subseteq T$. For a set of vertices, $\Gamma(\cdot)$ represents their 
neighbor-list. Since, the size of neighbor-list of a vertex-set is monotone and submodular, 
for $u \notin T$, we have $|\Gamma(S)| \le |\Gamma(T)|$, and 
$|\Gamma(S \cup \{u\})| - |\Gamma(S)| \ge |\Gamma(T \cup \{u\})| - |\Gamma(T)|$. 
By construction, for a feature set, $S$, $\Gamma(S)$ contains the entity-pairs for which 
at least one feature-value out of $S$ is different. Thus, $DistCnt$ function is 
$\frac{|\Gamma(\cdot)|}{|V_2|}$ and it is submodular.
\end{proof}

\begin{algorithm}
\renewcommand{\algorithmicrequire}{\textbf{Input:}}
\renewcommand{\algorithmicensure}{\textbf{Output:}}
\caption{Greedy Algorithm for $HamDist$}
\label{alg:2}
\begin{algorithmic}[1]
\REQUIRE $D (E, I)$, $k$
\ENSURE $S$
\STATE Sort the features in non-increasing order based on their $hamDist$, denoted as $F_{sorted}$
\STATE $S = \emptyset $
\FOR {each feature $x \in F_{sorted}$}
 \IF {$AC(\Pi_{S \cup \{ x \}}(D)) \geq k$}
   \STATE $S = S \cup \{ x \}$	 
 \ELSE \STATE \textbf{break}
 \ENDIF
\ENDFOR 
\STATE \textbf{return} $S$
\end{algorithmic}
\end{algorithm}

\begin{theorem}\label{th:approx}
For monotone submodular function $f$, let $S$ be a set of size $k$ obtained by
selecting elements one at a time, each time choosing an element provides
the largest marginal increase in the function value. Let $S^{\ast}$ be a set
that maximizes the value of $f$ over all $k$-element sets. Then $f(S) \ge (1 -
\frac{1}{e}) f(S^{\ast})$; in other words, $S$ provides $(1 -
\frac{1}{e})$-approximation. For modular function $f(S)=f(S^{\ast})$~\cite{Conforti.Cornuejols:84}.  \end{theorem}

\subsubsection{Greedy Method (Pseudo-code)} 

Using the above theorems we can design two greedy algorithms, one for modular
function $HamDist$, and the other for submodular function $DistCnt$. The
pseudo-codes of these algorithms are shown in  Algorithm~\ref{alg:2} and
Algorithm~\ref{alg:3}. Both the methods take binary dataset $D$ and integer
value $k$ as input and generate the selected feature set $S$ as output.
Initially $S = \emptyset$. For modular function, the marginal gain of an added
feature can be pre-computed, so Algorithm~\ref{alg:2} first sorts the features
in non-increasing order of their $HamDist$ values, and greedily adds features
until it encounters a feature such that its addition does not satisfy the $AC$
constraint.  For submodular function $DistCnt$, margin gain cannot be
pre-computed, so Algorithm~\ref{alg:3} selects the new feature by iterating over
all the features and finding the best one (Line 5 -11).  The terminating
condition of this method is also identical to the Algorithm~\ref{alg:2}.  Since
the number of features is finite, both the methods always terminate with a
valid $S$ which satisfies $AC(\Pi_S(D))\ge k$. 

Compared to \algf, both greedy methods are faster. With respect to
number of features ($d$), Algorithm~\ref{alg:2} runs in $O(d \lg d)$ time and
Algorithm~\ref{alg:3} runs in $O(d^2)$ time. Also, using
Theorem~\ref{th:approx}, Algorithm 2 returns the optimal size $|S|$
feature-set, and Algorithm 3 returns $S$, for which the objective function
value is $(1-1/e)$ optimal over all possible size-$|S|$ feature sets.

\begin{algorithm}
\renewcommand{\algorithmicrequire}{\textbf{Input:}}
\renewcommand{\algorithmicensure}{\textbf{Output:}}
\caption{Greedy Algorithm for $DistCnt$}
\label{alg:3}
\begin{algorithmic}[1]
\REQUIRE $D (E, I)$, $k$
\ENSURE $S$
\STATE $T = \emptyset $
\REPEAT
 \STATE $S=T$
 \STATE $\Delta H_{max} = 0.0$ 
 \FOR{$u \in I \setminus S$}
   \STATE Compute $\Delta H = DistCnt(S \cup \{u\}) - DistCnt(S)$
   \IF {$ \Delta H > \Delta H_{max} $}
     \STATE $\Delta H_{max} = \Delta H$
     \STATE $u_{m} = u$
   \ENDIF
 \ENDFOR
 \STATE $T = S \cup \{ u_{m} \} $
\UNTIL {$AC(\Pi_T(D)) \geq k$}
\STATE \textbf{return} $S$
\end{algorithmic}
\end{algorithm}

\begin{table}[t!]
\centering
\scalebox{1.00}{
\begin{tabular}{l c c c c c}
\toprule
Dataset & \# Entities & \# Features & \# Pos & \# Neg & Density \\  \midrule
Adult Data & 32561 & 19 & 24720 & 7841 & 27.9\% \\
Entity & 148 & 552 & 74 & 74 &  9.7\% \\
Disambiguation &  &  & &  &   \\
Email & 1099 & 24604 & 618 & 481 & 0.9\% \\ \bottomrule
\end{tabular}
}
\caption{Statistics of Real-World Datasets}
\label{tab:data}
\vspace{-0.1in}
\end{table}

%
%
%
%

%
%
%
%

\section{Experiments and Results}

In order to evaluate our proposed methods we perform various experiments. Our
main objective in these experiments is to validate how the performance of the
proposed privacy preserving classification varies as we change the value
of $AC$---user-defined privacy threshold metric. We also compare the
performance of our proposed utility preserving anonymization methods with other
existing anonymization methods such as $k$-anonymity and differential privacy.  It is important to note that we do not claim that our methods provide a better utility with identical privacy protection as other methods, rather we claim that our methods provide adequate privacy protection which is suitable for high dimensional sparse microdata with a much
superior AUC value---a classification utility metric which we want to maximize in our problem setup. We use three real-world datasets for our experiments. All three datasets consist of entities that are labeled with 2 classes. 
The number of entities, the number of features, the distribution of
the two classes (\#postive and \#negative), and the dataset density (fraction of non-zero cell values) 
are shown in Table~\ref{tab:data}. 

\subsection{Privacy Preserving Classification Tasks}

Below, we discuss the datasets and the privacy preserving classification tasks
that we solve using our proposed methods.\\

\noindent {\bf Entity Disambiguation (ED)~\cite{zhang2016bayesian}.} The objective of this
classification task is to identify whether the name reference at a row in the
data matrix maps to multiple real-life persons or not.  Such an exercise is
quite common in the Homeland Security for disambiguating multiple
suspects from their digital footprints~\cite{zhang2014name,Zhang.Saha.ea:15}. 
Privacy of the people in such a dataset
is important as many innocent persons can also be listed as a
suspect. Given a set of keywords that are associated with a name reference, we
build a binary data matrix for solving the ED task. We use
Arnetminer\footnote{\url{http://arnetminer.org}} academic publication data.  In
this dataset, each row is a name reference of one or multiple researchers, and
each column is a research keyword within the computer science research umbrella. A
`1' entry represents that the name reference in the corresponding row has
used the keyword in her (or their) published works. 
In our dataset, there are 148 rows which are labeled such that half of the people in this dataset are pure 
entity (a negative case), and the rest of them are multi-entity (a positive case). 
The dataset contains 552 attributes (keywords).

To solve the entity disambiguation problem we first perform topic modeling over
the keywords and then compute the distribution of entity $u$'s keywords across
different topics.  Our hypothesis is that for a pure entity the topic
distribution will be concentrated on a few related topics, but for an impure
entity (which is composed of multiple real-life persons) the topic distribution
will be distributed over many non-related topics. We use this idea to build a
simple classifier which uses an entropy-based score $E(u)$ for an entity $u$ as below:

\begin{equation}
\label{eq:entropy}
E(u) = - \sum_{k=1}^{\mid T \mid} P(u \mid T_k) \log P(u \mid T_k)
\end{equation}

where $P(u \mid T_k)$ is the probability of $u$ belonging to topic $T_k$, and
$|T|$ represents the pre-defined number of topics for topic modeling. Clearly,
for a pure entity the entropy-based score $E(u)$ is relatively smaller than the
same for a non-pure entity. We use this score as our predicted value and
compute AUC (area under ROC curve) to report the performance of the classifier.

\noindent {\bf Adult.} The Adult dataset
\footnote{\url{https://archive.ics.uci.edu/ml/datasets/Adult}} 
is based on census data and has been
widely used in earlier works on $k$-anonymization~\cite{Iyengar:02}.  For our
experiments, we use eight of the original attributes; these are age, work
class, education, marital status, occupation, race, gender, and hours-per-week.
The classification task is to determine whether a person earns over 50K a year
or not. Among all of the attributes, gender is originally binary. For the other
attributes, we make them as binary for our purpose. For example, for marital
attribute, we consider never-married (1) versus others (0). For race attribute,
we consider white (1) versus others (0). For the numerical attributes, we cut
them into different categories and consider a binary attribute for each
category. For instance, we partition age value in five non-overlapping
intervals: $[0, 25]$, $(25, 35]$, $(35, 45]$, $(45,55]$, and $(55, \infty]$,
and then each of the five intervals becomes a binary attribute. Similarly,
education attribute is divided into $4$ intervals and hour/week attribute is
divided into $5$ interval. 
In this way, we have a total of $19$ attributes for
the Adult dataset. As we can see privacy of the individuals in such a dataset is
quite important as many people consider their personal data, such as
race, gender, marital status and so on as sensitive attributes and they are not
willing to release them to public.

\noindent {\bf Email} The last dataset, namely Email dataset
\footnote{\url{http://www.csmining.org/index.php/pu1-and-pu123a-datasets.html}}
is a collection of approximately 1099 personal email messages distributed in 10
different directories. Each directory contains both legitimate and spam
messages. To respect the privacy issue, each token including word, number, and
punctuation symbol is encrypted by a unique number.  The classification task is
to distinguish the spam email with nonspam email. We use this data to mimic
microdata (such as twitter or Facebook messages) classification. Privacy is
important in such a dataset as keyword based features in a micro-message can
potentially identify a person. In the dataset, each row is an email message,
and each column denotes a token. A `1' in a cell represents that the row
reference contains the token in the email message. 

\subsection{Experimental Setting}~\label{sec:5.2}

For our experiments, we vary the $k$ value of the proposed $k$-anonymity by
selection ($AC$) metric and run \algf\ and different variants of \algs\
independently for building projected classification datasets for which $AC$
value is at least $k$. We use the names $HamDist$ and $DistCnt$ for the two
variants of \algs\ (Algorithm~\ref{alg:2} and~\ref{alg:3}), which optimize
Hamming distance and Distinguish count greedy criteria, respectively. As we
mentioned earlier, $k$-anonymity based method imposes strong restriction which
severely affects the utility of the dataset. To demonstrate that, instead of
using $AC$, we utilize $k$-anonymity as our privacy criteria for different
variants of \algs. We call these competing methods $k$-anonymity $HamDist$, and
$k$-anonymity $DistCnt$. It is important to note that, in our experiments under the same $k$ setting, the $k$-anonymity based competing methods may not provide the same  level of privacy. For instance,  for the same $k$ value,
privacy protection of our proposed method $HamDist$ is  not the same as  that of the $k$-anonymity $HamDist$, 
simply because $k$-AC is a  relaxation of $k$-anonymity.

We also use four other methods for comparing their
performance with the performance of our proposed methods. We call these
competing methods 
RF~\cite{Byun.Bertino.ea:07}, $CM$ Greedy~\cite{Iyengar:02}, Laplace-DP, and Exponential-DP. We discuss these methods below.

RF is a Randomization Flipping based $k$-anonymization technique presented
in~\cite{Byun.Bertino.ea:07}, which randomly flips the feature value such that
each instance in the dataset satisfies the $k$-anonymity privacy constraint. RF
uses clustering such that after random flipping operation, each cluster has at
least $k$ entities with the same feature values with respect to the entire
feature set. 

CM greedy represents another greedy
based method which uses Classification Metric utility criterion proposed
in~\cite{Iyengar:02} as utility metric (See definition~\ref{def:cm}).  
It assigns a generalization penalty over the rows of the dataset and uses a genetic algorithm 
for the classification task, but for a fair comparison we use CM criterion in the \algs\ algorithm 
and with the selected features we use identical setup for classification. 

Laplace-DP~\cite{Jafer.Stan.ea:14} is a method to use feature selection for $\epsilon$-differential private data publishing. 
Authors in~\cite{Jafer.Stan.ea:14} utilize Laplace mechanism~\cite{Dwork.Frank.ea:06} for $\epsilon$-differential privacy guarantee.
To compare with their method, we first compute the utility of each feature $x_{i} \in I$ as 
its true output using $HamDist$ function in Definition~\ref{def:hamdist} denoted as $H(x_{i})$. 
Then we add independently generated noise according to a Laplace distribution with $Lap(\frac{\Delta H}{\epsilon})$ to each of the $|I|$ outputs, and
the noisy output for each feature $x_{i}$ is defined as $\hat{H(x_{i})} = H(x_{i}) + Lap(\frac{\Delta H}{\epsilon})$, where $\Delta H$ is the sensitivity of $HamDist$ function. 
After that we select top-$N$ features by considering $N$ largest noisy outputs.  
On the reduced dataset, we apply a private data release method which provides $\epsilon$-differential privacy guaranty. The
general philosophy of this method is to first derive a frequency matrix of the
reduced dataset over the feature domain and add Laplace noise with $Lap(\frac{1}{\epsilon})$ to each count (known as
marginal) to satisfy the $\epsilon$-differential privacy. Then the method
adds additional data instances to match the above count. Such an approach is
discussed in~\cite{Dwork:08} as a private data release mechanism.

Exponential-DP is another $\epsilon$-differential privacy aware feature selection method. Compared to the work presented in~\cite{Jafer.Stan.ea:14},
we use exponential mechanism~\cite{Frank.Kunal:07} based $\epsilon$-differential privacy to select features.
In particular, we choose each feature $x_{i} \in I$ with probability proportional to $exp(\frac{\epsilon}{2\Delta H}H(x_{i}))$. That is, 
the feature with a higher utility score in terms of $HamDist$ function is exponentially more likely to be chosen. The private 
data release stage of Exponential-DP is as same as the one in Laplace-DP. Note that, for both Laplace-DP and Exponential-DP, prior feature selection is essential for such
methods to reduce the data dimensionality, otherwise the number of marginals is
an intractable number ($2^{|{\cal F}|}$, for a binary dataset with ${\cal F}$
features) and adding instances to match count for each such instance is practically impossible.

For all the algorithms and all the datasets (except ED) we use the LibSVM to
perform SVM classification using L2 loss with 5-fold cross validation. The only
parameter for libSVM is regularization-loss trade-off $C$ which we tune using a
small validation set.  For each of the algorithms, we report AUC and the
selected feature count (SFC).  For RF method, it selects all the features, so
for this method we report the percentage of cell values for which the bit is
flipped.
We use different $k$-anonymity by containment ($AC$) values in our experiments.
For practical $k$-anonymization, $k$ value between 5 and 10 is suggested in
the earlier work\cite{Sweeney:02}; we use three different $k$ values, which are
$5, 8$ and $11$. For a fair comparison, for both Laplace and Exponential DP, we use the same number of features
as is obtained for the case of $HamDist$ Greedy under $k=5$. 
Since $k$-anonymity and differential privacy use totally different parameter setting mechanisms (one based on $k$, and the other based on $\epsilon$), it is not easy to understand what value of $\epsilon$ in DP will make a fair comparison 
for a $k$ value of 5 in $k$-AC. So, 
for both Laplace-DP and exponential-DP,
we show the differential privacy results for
different $\epsilon$ values: $0.5, 1.0, 1.5$, and $2.0$ . Note that the original work~\cite{Dwork:08} has
suggested to use a value of 1.0 for $\epsilon$. While using DP based methods,
we distribute half of the privacy budget for the feature selection step and the remaining half to add noise into marginals in the private data release step. 
Moreover, in the feature selection procedure, we further equally divide the budget for the selection of each feature.

RF, Laplace-DP, and Exponential-DP are randomized methods, so for each dataset we run all
of them $10$ times and report the average AUC and standard deviation. For each
result table in the following sections, we also highlight the best results in
terms of AUC among all methods under same $k$ setting. We run all the experiments 
on a 2.1 GHz Machine with 4GB memory running Linux operating system.

\begin{table*}[t!]
\centering
\scalebox{1.00}{
\begin{tabular}{l c c c}
\toprule
\multirow{2}{*}{Method} &  \multicolumn{3}{c}{AUC (Selected Feature Count)}  \\ \cline{2-4}
{} & k=5 & k=8 & k=11 \\ \midrule
\algf & 0.82 (61) & 0.81 (43) & 0.79 (32)  \\
$HamDist$ & {\bf 0.88 (27)} & {\bf 0.88 (24)} & {\bf 0.81 (16)} \\
$DistCnt$ & 0.81 (11) & 0.81 (11) & 0.80 (10) \\ \midrule
CM Greedy ~\cite{Iyengar:02} & 0.68 (2) & 0.68 (2) & 0.68 (2) \\
RF ~\cite{Byun.Bertino.ea:07} & 0.75$\pm$0.02 (11.99\%) & 0.73$\pm$0.03 (14.03\%) & 0.72$\pm$0.02 (16.49\%) \\ 
$k$-anonymity $HamDist$ & 0.55 (3) & 0.55 (2) & 0.55 (2) \\
$k$-anonymity $DistCnt$ & 0.79 (3) & 0.79 (3) & 0.77 (2) \\ \midrule
Full-Feature-Set   & \multicolumn{3}{c} {0.87 (552)} \\
\bottomrule
\end{tabular}}
\caption{AUC Comparison among different privacy methods for the name entity disambiguation task}
\label{tab:NED}
\end{table*}

\subsection{Name Entity Disambiguation}

In Table~\ref{tab:NED} we report the AUC value of anonymized name entity
disambiguation task using various privacy methods (in rows) for different $k$ values
(in columns). For better comparison, our proposed methods, competing methods,
and non-private methods are grouped by the horizontal lines: our proposed
methods are in the top group, the competing methods are in the middle group,
and non-private methods are in the bottom group.  
For differential privacy comparison, we show the AUC result in Figure~\ref{fig:1}~\ref{fig:2}. 
For each method, we also report the count of selected features (SFC). Since RF method uses
the full set of features; for this method the value in the parenthesis is the
percent of cell values that have been flipped. We also report the AUC
performance using full feature set (last row). As non-private method in bottom 
group has no privacy restriction, thus the result is independent of $k$.

For most of the methods increasing $k$ decreases the number of
selected features, which translates to  poorer classification performance; this
validates the privacy-utility trade-off. However, for a given $k$, our proposed
methods perform better than the competing methods in terms of AUC metric for
all different $k$ values.  For instance, for $k=5$, the AUC result of RF and CM
Greedy are only $0.75$ and $0.68$ respectively, whereas different versions of proposed
\algs\ obtain AUC values between $0.81$ and $0.88$. Among the competing
methods, both Laplace-DP and Exponential-DP perform the worst ($0.51$ AUC under $\epsilon = 1.0$) 
as shown in the first group of bars in Figure~\ref{fig:1} \& \ref{fig:2}, and $k$-anonymity
$DistCnt$ performs the best (0.79 for $k$=5); yet all completing methods perform much poorer than
our proposed methods. A reason for this may be most of the competing methods are too
restrictive, as we can see that they are able to select only 2 to 3 features
for various $k$ values. In comparison, our proposed methods are able to select between 11 and 61 features, which help our methods
to retain classification utility. The bad performance of differential privacy based methods is due to the fact that in such a setting, the added
noise is too large in both feature selection and private data release steps. 
In general, the smaller the $\epsilon$, the stronger privacy guarantee the differential privacy provides. However, stronger privacy protection in terms of $\epsilon$ always leads to worse data utility in terms of AUC as shown in Figure~\ref{fig:1}~\ref{fig:2}. Therefore, even though differential privacy provides stronger privacy guarantee, the utility of data targeting supervised classification task is significantly destroyed.
For this dataset, we observe that the performance of RF is largely dependent on the percentage of flips in the
cell-value; if this percentage is large, the performance is poor. As $k$
increases, with more privacy requirement, the percentage of flips increases,
and the AUC drops.

For a sparse dataset like the one that we use for entity disambiguation,
feature selection helps classification performance. In this dataset, using full
set of features (no privacy), we obtain only 0.87 AUC value, whereas using less
than 10\% of features we can achieve comparable or better AUC using our
proposed methods (when $k$=5). Even for $k=11$, our methods retain substantial
part of the classification utility of the dataset and obtain AUC value of 0.81
(see second row). 
Also, note that under $k=5$ and $8$, our 
$HamDist$ performs better than using full feature set,
which demonstrates our proposed privacy-aware feature selection methods not only have the competitive AUC performance, 
but provide strong privacy guarantees as well.

\begin{table*}[t!]
\centering
\scalebox{1.00}{
\begin{tabular}{l c c c}
\toprule
\multirow{2}{*}{Method} &  \multicolumn{3}{c}{AUC (Selected Feature Count)}  \\ \cline{2-4}
{} & k=5 & k=8 & k=11 \\ \midrule
\algf & 0.74 (8) & 0.74 (8) & 0.75 (7)  \\
$HamDist$ & 0.77 (9) & 0.77 (9) & 0.76 (8) \\
$DistCnt$ & 0.78 (10) & 0.78 (10) & 0.76 (8) \\\midrule
CM Greedy ~\cite{Iyengar:02} & 0.71 (5) & 0.71 (5) & 0.71 (5) \\
RF ~\cite{Byun.Bertino.ea:07} & {\bf 0.80$\pm$0.02 (0.60\%)} & {\bf 0.80$\pm$0.03 (1.00\%)} & {\bf 0.80$\pm$0.02 (1.44\%)} \\ 
$k$-anonymity $HamDist$ & 0.72 (8) & 0.72 (8) & 0.72 (8) \\
$k$-anonymity $DistCnt$ & 0.73 (8) & 0.70 (6) & 0.70 (6) \\ \midrule
Full-Feature-Set & \multicolumn{3}{c} {0.82 (19)} \\
\bottomrule
\end{tabular}}
\caption{Comparison among different privacy methods for Adult dataset using AUC}
\label{tab:Ad}
\end{table*}

\begin{table*}[t!]
\centering
\scalebox{1.00}{
\begin{tabular}{l c c c}
\toprule
\multirow{2}{*}{Method} & \multicolumn{3}{c}{AUC (Selected Feature Count)}  \\ \cline{2-4}
{} & k=5 & k=8 & k=11 \\ \midrule
\algf & 0.94 (121) & 0.92 (66) & 0.90 (58)  \\
$HamDist$ & 0.91 (11) & 0.91 (11) & 0.91 (11) \\
$DistCnt$ & {\bf 0.95 (11)} & {\bf 0.93 (7)} & {\bf 0.93 (7)} \\\midrule
CM Greedy ~\cite{Iyengar:02} & 0.86 (3) & 0.86 (3) & 0.86 (3) \\
RF ~\cite{Byun.Bertino.ea:07} & 0.87$\pm$0.02 (1.30\%) & 0.86$\pm$0.01 (1.73\%) & 0.87$\pm$0.03 (2.03\%) \\
$k$-anonymity $HamDist$ & 0.84 (4) & 0.84 (4) & 0.84 (4) \\
$k$-anonymity $DistCnt$ & 0.81 (4) & 0.81 (4) & 0.81 (4) \\ \midrule 
Full-Feature-Set & \multicolumn{3}{c} {0.95 (24604)} \\
\bottomrule
\end{tabular}}
\caption{Comparison among different privacy methods for Email dataset using AUC}
\label{tab:email}
\end{table*}

\begin{figure*}[h]
\centering
\subfigure[Laplace Mechanism based Differential Privacy]
{
\label{fig:1}
\includegraphics[height=0.36\linewidth, angle=0] {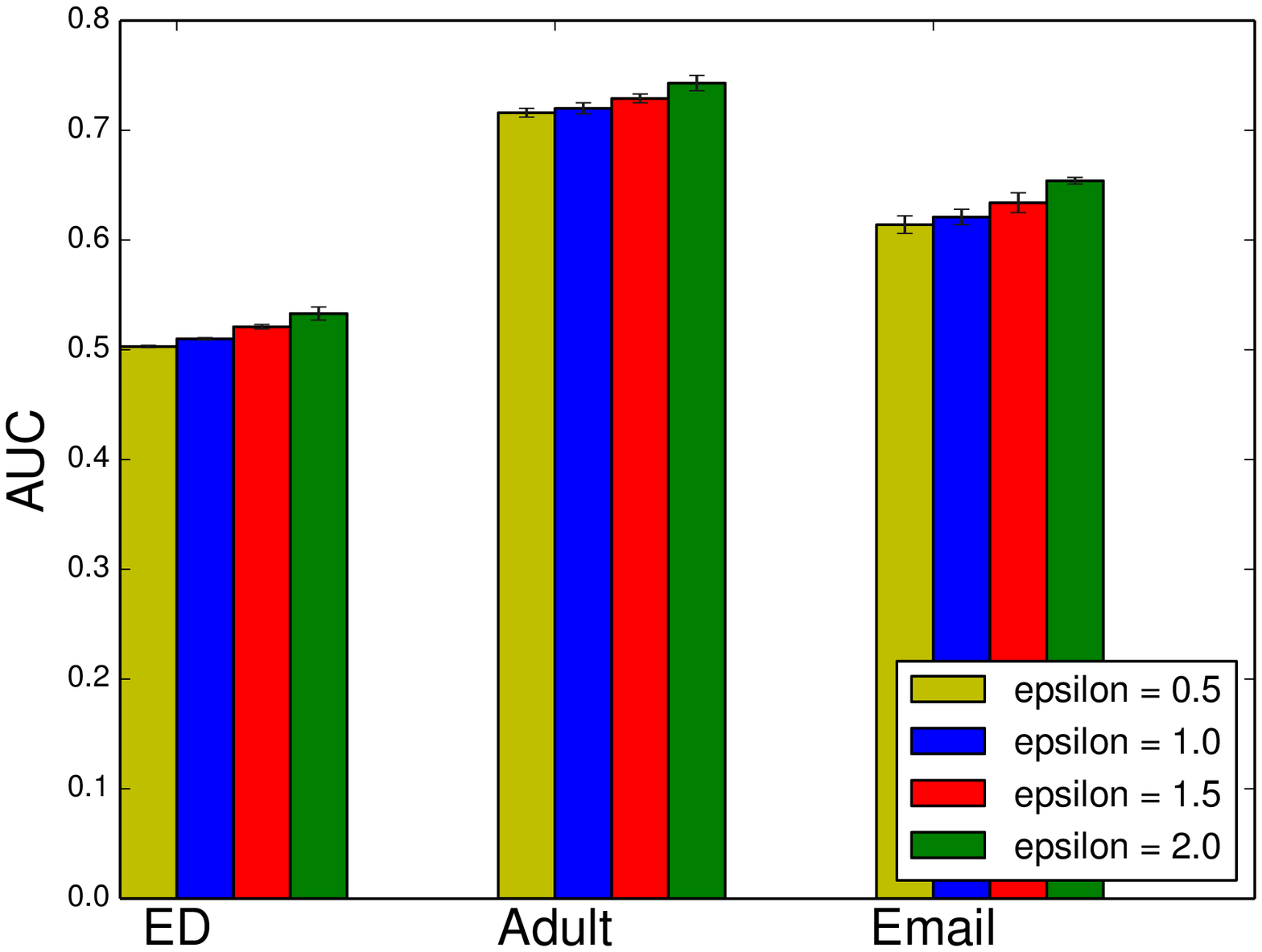}
}
\subfigure[Exponential Mechanism based Differential Privacy]
{
\label{fig:2}
\includegraphics[height=0.36\linewidth, angle=0] {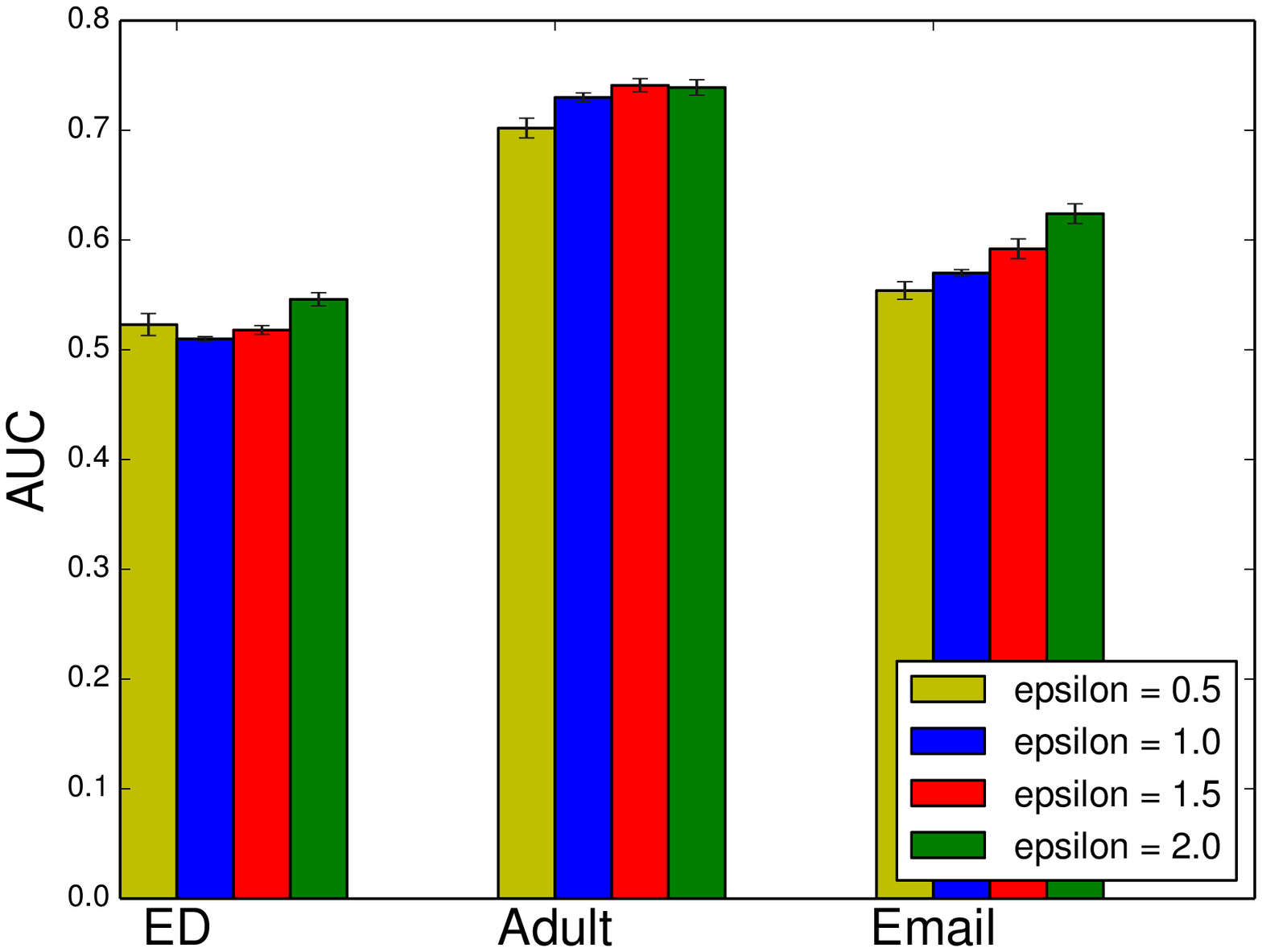}
}
\caption{Classification Performance of differential privacy based methods for different $\epsilon$ values 
on three datasets. Laplace mechanism is on the left and Exponential mechanism is on the right. Each group of bars belong to one
specific dataset, and within a group different bars represent different $\epsilon$ values.}
\vspace{-0.05in}
\end{figure*}

\subsection{Adult Data}

The performance of various methods for the Adult dataset is shown in
Table~\ref{tab:Ad}, where the rows and columns are organized identically as in
the previous table. Adult dataset is low-dimensional and dense (27.9\% values are
non-zero). Achieving privacy on such a dataset is comparatively easier, so
existing methods for anonymization work fairly well on this. As we can
observe, RF performs the best among all the methods. The good performance of RF
is owing to the very small percentage of flips which ranges from 0.60\% to
1.50\% for various $k$ values.  Basically, RF can achieve $k$-anonymity on this
dataset with very small number of flips, which helps it maintain the
classification utility of the dataset.  For the same reason, $k$-anonymity
$HamDist$ and $k$-anonymity $DistCnt$ methods are also able to retain many
dimensions (8 out of 19 for $k$=5) of this dataset, and perform fairly well. On
the other hand, different versions of \algs\ and \algf\ retain between 8 and 10
dimensions and achieve an AUC between 0.74 and 0.78, which are close to 0.80
(the AUC value for RF). Also, note that, the AUC using the full set of features
(no privacy) is $0.82$, so the utility loss due to the privacy is not
substantial for this dataset. As a remark, our method is particularly suitable for high
dimensional sparse data for which anonymization using traditional methods is
difficult. 

\subsection{Spam Email Filtering}

In Table~\ref{tab:email}, we compare AUC value of different methods for spam
email filtering task.  This is a very high dimensional data with $24604$
features. As we can observe, our proposed methods, especially $DistCnt$ and
$HamDist$ perform better than the competing methods. For example, for $k=5$,
the classification AUC of RF is $0.87$ with flip rate $1.30\%$, but using less
than $0.045\%$ of features $DistCnt$ obtains an AUC value of $0.95$, which is
equal to the AUC value using the full feature set. Again, $k$-anonymity based
methods show worse performance as they select less number of features due to
stronger restriction of this privacy metric.  For instance, for $k=5$,
$HamDist$ selects $11$ features, but $k$-anonymity $HamDist$ selects only 4
features. Due to this, classification results using $k$-anonymity constraint
are worse compared to those using our proposed $AC$ as privacy metric. As shown in Figure~\ref{fig:1}~\ref{fig:2}, 
both Laplace-DP and Exponential-DP with various privacy budget $\epsilon$ setups
perform much worse than all the competing methods in Table~\ref{tab:email}, 
which demonstrates that the significant amount of added noise during
the sanitization process deteriorates the data utility and leads to bad classification performance. Among
our methods, both $DistCnt$ and \algf\ are the best as they consistently hold
the classification performance for all different $k$ settings.

\section{Related Work}

We discuss the related work under the following two categories.

\subsection{Privacy-Preserving Data Mining}

In terms of privacy model, several privacy metrics have been widely used in
order to quantify the privacy risk of published data instances, such as
$k$-anonymity~\cite{Sweeney:02}, $t$-closeness~\cite{Li.Li:07},
$\ell$-diversity~\cite{Machanavajjhala.Kifer.ea:07}, 
and differential privacy~\cite{Dwork.Roth:14}.
Existing works on privacy preserving data mining
solve a specific data mining problem given a privacy constraint over the data
instances, such as classification~\cite{Vaidya.Clifton.ea:08},
regression~\cite{Fienberg.Jin:12},
clustering~\cite{Vaidya.Clifton:03} and frequent pattern
mining~\cite{Evfimievski.Agrawal.ea:02}.  
However the solutions proposed in these works are strongly influenced by the specific data
mining task and also by the specific privacy model. 
In fact, the majority of the above works consider distributed privacy
where the dataset is partitioned among multiple participants owning different
portions of the data, and the goal is to mine shared insights over the global
data without compromising the privacy of the local portions. 
A few other works~\cite{Bonomi.Xiong:13,Bhaskar:Laxman.ea:10} consider output privacy by
ensuring that the output of a data mining task does not
reveal sensitive information.

$k$-anonymity privacy metric, due to its simplicity and flexibility, has been
studied extensively over the years. Authors in~\cite{Atzori.Bonchi.ea:05} presents
the $k$-anonymity patterns for the application of association rule mining. Samarati~\cite{Samarati:01} proposes
formal methods of $k$-anonymity using suppression and generalization
techniques. She also introduced the concept of minimal generalization.
Meyerson et al.\cite{Meyerson.Williams:04} prove that two definitions of
$k$-optimality are ${\cal NP}$-hard: first, to find the minimum number of cell
values that need to be suppressed; second, to find the minimum number of
attributes that need to be suppressed. Henceforth, a large number of works
have explored the approximation of
anonymization~\cite{Meyerson.Williams:04,Bayardo.Agrawal:05}. However, none of
these works consider the utility of the dataset along with the privacy
requirements. Kifer et al.\cite{Kifer.Gehrke:06} propose methods that inject utility in the form of
data distribution information into $k$-anonymous and $\ell$-diverse tables.
However, the above work does not consider a classification dataset.
Iyengar~\cite{Iyengar:02} proposes a utility metric called $CM$ which is
explicitly designed for a classification dataset. However, It assigns a generalization
penalty over the rows of the dataset, but its performance is poor as we have shown in this work. 

In recent years differential privacy~\cite{Dwork:08,nguyen:hal-01424911,Lee:2012:DI:2339530.2339695} has attracted much attention in privacy
research literatures. Authors in ~\cite{Chen.Xiao.ea:15} propose a sampling based method for releasing
high dimensional data with differential privacy guarantees. ~\cite{Qardaji.Yang.ea:14} proposes a method of publishing 
differential private low order marginals for high dimensional data. Even though authors in~\cite{Chen.Xiao.ea:15, Qardaji.Yang.ea:14}
claim that they deal with high dimensional data, the dimensionality of data is at most $60$ from the experiments in their works. 
~\cite{Sanchez.Martinez.ea:14} makes use of $k$-anonymity to enhance data
utility in differential privacy. An interesting observation of this work is that 
differential privacy based method, by itself, is not a good privacy mechanism, in regards
to maintaining data utility. ~\cite{Chen.Xiong.ea:11} proposes a probabilistic top-down partitioning algorithm
to publish the set-valued data via differential privacy. 
Authors of~\cite{Mohammed.Chen.ea:11} propose to utilize exponential mechanism to release
a decision tree based classifier that satisfies $\epsilon$-differential privacy. 
However in their work, privacy is embedded in the data mining process, hence
they are not suitable as a data release mechanism, and more importantly they can only
be used along with the specific classification model within which the privacy mechanism
is built-in. 

\subsection{Privacy-Aware Feature Selection}

Empirical study for the use of feature selection in Privacy Preserving Data
Publishing has been proposed in~\cite{Yasser.Stan.ea:14}
~\cite{Jafer.Stan.ea:14}. However, in their work, they use feature selection as
an add-on tool prior to data anonymization and do not consider privacy during
the feature selection process. For our work, we consider privacy-aware feature
selection with a twin objective of privacy preservation and utility maintenance. 
To the best of our knowledge, the most similar works to ours for the use of feature selection in 
Privacy Preserving Data Publishing are presented in~\cite{Pattuk.Malin.ea:15, Matatov.Rokach.ea:10} recently. ~\cite{Pattuk.Malin.ea:15} considers privacy as a
cost metric in a dynamic feature selection process and proposes a greedy
based iterative approach for solving the task, where the data releaser requests information about
one feature at a time until a predefined privacy budget is exhausted.  
However the entropy based privacy metric presented in this work is strongly influenced
by the specific classifier. ~\cite{Matatov.Rokach.ea:10} presents a genetic approach
for achieving $k$-anonymity by partitioning the original dataset into several projections such that each one of them adheres to $k$-anonymity.
But the proposed method does not provide optimality guaranty. 

\section{Conclusion and Future Work}

In this paper, we propose a novel method for entity anonymization using feature
selection. We define a new anonymity metric called $k$-anonymity by containment
which is particularly suitable for high dimensional microdata. We also propose two
feature selection methods along with two classification utility metrics. These
metrics satisfy submodular properties, thus they enable effective greedy algorithms.
In experiment section we show that both proposed methods select good quality features 
on a variety of datasets for retaining the classification utility yet they satisfy the 
user defined anonymity constraint. 

In this work, we consider binary features. We also show experimental results
using categorical features by making them binary, so the work can easily be
extended for datasets with categorical features. 
An immediate future work is to
extend this work on datasets with real-valued features. 
Another future direction would be to consider absent attributes in the privacy model.
In real world, for some binary datasets, absent attributes can cause privacy violation, such as, they can
be used for negative association rule mining.

\section*{Acknowledgements}

We sincerely thank the reviewers for their insightful comments. This research is sponsored by both Mohammad Al Hasan's
NSF CAREER Award (IIS-1149851) and Noman Mohammed's NSERC Discovery Grants (RGPIN-2015-04147). The contents are
solely the responsibility of the authors and do not necessarily represent the official views of NSF and NSERC.

\bibliographystyle{abbrv}
\balance
\bibliography{privacy} 
\end{document}